\documentclass[10pt,twocolumn,letterpaper]{article}

\usepackage{iccv}
\usepackage{times}
\usepackage{epsfig}
\usepackage{graphicx}
\usepackage{amsmath}
\usepackage{amssymb}
\usepackage{mathtools}
\usepackage{amsthm}
\usepackage{multirow}
\usepackage{booktabs}
\usepackage{pifont}
\usepackage{caption}
\usepackage{subcaption}
\usepackage{tabularx}
\usepackage{xinttools}
\usepackage{array}
\usepackage[super]{nth}

\newcommand{\cmark}{\ding{51}}%
\newcommand{\xmark}{\ding{55}}%
\renewcommand{\arraystretch}{1.06}

\DeclareMathOperator*{\argmax}{arg\,max}
\DeclarePairedDelimiterX{\infdivx}[2]{(}{)}{#1\;\delimsize|\delimsize|\;#2}

\newtheorem{theorem}{Theorem}

\renewcommand{\S}       {\mathcal{S}}
\newcommand{\T}         {\mathcal{T}}
\newcommand{\Y}         {\mathcal{Y}}
\newcommand{\N}         {\mathcal{N}}
\newcommand{\LL}         {\mathcal{L}}
\newcommand{\D}         {\pmb D}

\newcommand{\x}         {\pmb x}
\newcommand{\z}         {\pmb z}
\newcommand{\y}         {\pmb y}
\renewcommand{\ll}         {\pmb l}
\newcommand{\yt}        {\tilde{\pmb y}}
\newcommand{\yh}        {\bar{\pmb y}}
\newcommand{\pimg}        {p_{img}}
\newcommand{\ploc}        {p_{loc}}
\newcommand{\pcls}        {p_{cls}}
\newcommand{\pinit}        {p_{init}}

\newcommand{\bmu}      {\pmb \mu}
\newcommand{\bSigma}      {\pmb \Sigma}
\newcommand{\eye}      {\pmb I}
\newcommand{\KL}        {\text{KL}}
\newcommand{\simk}    {\textit{SIM 10K}}
\newcommand{\city}    {\textit{Cityscapes}}
\newcommand{\cityfog} {\textit{Foggy Cityscapes}}
\newcommand{\kit} {\textit{KITTI}}

\newcommand{\faster} { Faster R-CNN\cite{ren2015faster} }
\newcommand{\mani} { Pseudo-labeling\cite{inoue2018cross} }
\newcommand{\adaptive} { Feature Learning~\cite{chen2018domain} }
\newcommand{\ours} { Noisy Labeling (Ours) }

\usepackage[pagebackref=true,breaklinks=true,letterpaper=true,colorlinks,bookmarks=false]{hyperref}

 \iccvfinalcopy % *** Uncomment this line for the final submission

\def\iccvPaperID{1356} % *** Enter the ICCV Paper ID here

\ificcvfinal\fi
\begin{document}

%%%%%%%%% TITLE

\title{A Robust Learning Approach to Domain Adaptive Object Detection}

\author{Mehran Khodabandeh\\
Simon Fraser University\\
Burnaby, Canada\\
{\tt\small mkhodaba@sfu.ca}
\and
Arash Vahdat\\
NVIDIA\\
California, USA\\
{\tt\small avahdat@nvidia.com}
\and
Mani Ranjbar\\
Quadrant\\
Burnaby, Canada\\
{\tt\small mani@quadrant.ai}
\and
William G. Macready\\
Quadrant\\
Burnaby, Canada\\
{\tt\small bill@quadrant.ai}
}

\maketitle
\ificcvfinal\thispagestyle{empty}\fi

\begin{abstract}
Domain shift is unavoidable in real-world applications of object detection. For example, in self-driving cars, the target domain consists of unconstrained road environments which cannot all possibly be observed in training data. Similarly, in surveillance applications sufficiently representative training data may be lacking due to privacy regulations.
In this paper, we address the domain adaptation problem from the perspective of robust learning and show that the problem may be formulated as training with noisy labels. We propose a robust object detection framework that is resilient to noise in bounding box class labels, locations and size annotations. To adapt to the domain shift, the model is trained on the target domain using a set of noisy object bounding boxes that are obtained by a detection model trained only in the source domain. We evaluate the accuracy of our approach in various source/target domain pairs and demonstrate that the model significantly improves the state-of-the-art on multiple domain adaptation scenarios on the SIM10K, Cityscapes and KITTI datasets.

\iffalse
   In this work we propose a method to address the challenging problem of domain adaptation for object detection. Although, object detection
   for common objects such as humans, cars, bicycles, etc. is considered a solved problem when a lot of labeled data is available in the
   target domain, the performance is still far from perfect when labeled data for target domain is scarce. This is certainly the case for many
   real applications including self driving cars, where the target domain is an uncontrolled environment, and object detection in security camera footage, where privacy concerns may prevent annotation of data in the target domain, etc.
   
   Our novel approach to tackle this problem can be summarized in four steps: first, training the detection module using available labeled data, second, generating noisy annotations for images in the target domain using the trained detector, 
   third, refining the machine generated annotations using a classification module, and fourth, retraining the detector using the original labeled data and refined machine generated labeled data from the target domain. The retraining step takes into account the possibility of mislabels in machine generated annotations.

  \fi
\end{abstract}

%%%%%%%%% BODY TEXT
\section{Introduction} 
Object detection lies at the core of computer vision and finds application in surveillance, medical imaging, self-driving cars, face analysis, and industrial manufacturing. 
Recent advances in object detection using convolutional neural networks (CNNs) 
have made current models fast, reliable and accurate.

However, domain adaptation remains a significant challenge in object detection.
In many discriminative problems (including object detection) it is usually assumed that the distribution of instances in both train (source domain) and test (target domain) set are identical. Unfortunately, 
this assumption is easily violated, and domain changes in object detection arise with variations in viewpoint, background, object appearance, scene type and illumination. Further, object detection models are often deployed in environments which differ from the training environment.

Common domain adaptation approaches are based on either supervised model fine-tuning in the target domain or unsupervised cross-domain representation learning. While 
the former requires additional labeled instances in the target domain, the latter eliminates this requirement at the cost of two new challenges. Firstly, the source/target representations 
should be matched in some space (e.g., either in input space~\cite{zhu2017unpaired, hoffman2017cycada} or hidden representations space~\cite{ganin2016domain, tzeng2017adversarial}). Secondly, a mechanism for feature matching must be defined (\textit{e.g.} maximum mean discrepancy (MMD)~\cite{murez2018image, long2017deep}, $\mathcal{H}$ divergence~\cite{chen2018domain}, or adversarial learning).

In this paper, we approach domain adaptation differently, and address the problem through robust training methods. Our approach relies on the observation that, although a (primary) model trained
in the source domain may have suboptimal performance in the target domain, it may nevertheless 
be used to detect objects in the target domain with some accuracy. The detected objects can 
then be used to retrain a detection model on both domains.
However, because the instances detected in the target domain may be inaccurate,
a robust detection framework (which accommodates these inaccuracies) must be used during retraining.

The principal benefit of this formulation is that the detection model is trained
in an unsupervised manner in the target domain. Although we do not explicitly
aim at matching representations between source and target domain, the detection
model may implicitly achieve this because it is fed by instances from both source and target domains.

To accommodate labeling inaccuracies we adopt a probabilistic perspective and develop a robust training framework for object detection on top of Faster R-CNN~\cite{ren2015faster}. 
We provide robustness against two types of noise: i) mistakes in object labels (\textit{i.e.}, a bounding box is labeled as person but actually is a pole), and 
ii) inaccurate bounding box location and size (\textit{i.e.}, a bounding box does not enclose the object).
We formulate the robust retraining objective so that the model can alter both bounding box class labels and bounding box location/size based on its current belief of 
labels in the target domain. This enables the robust detection model to refine the noisy labels in the target domain.

To further improve label quality in the target domain, we introduce an auxiliary image classification 
model. We expect that an auxiliary classifier can improve target domain labels because it may use cues that have not been utilized by the original detection model. 
As examples, additional cues can be based on additional input data (\textit{e.g.} motion or optical flow), different network architectures, or ensembles of models. 
We note however, that the auxiliary image classification model is only used during the retraining phase and the computational complexity of the final detector is preserved at test time.

The contributions of this paper are summarized as
follows: i) We provide the first (to the best of our knowledge) formulation of domain adaptation 
in object detection as robust learning.
ii)  We propose a novel robust object detection framework that
considers noise in training data on both object labels and locations. We use \faster as our base object detector, but our general framework, theoretically, could be adapted to other detectors (e.g. SSD~\cite{liu2016ssd} and YOLO~\cite{redmon2017yolo9000}) that minimize a classification loss and regress bounding boxes.
iii) We use an independent classification refinement module to allow other sources of information from the target domain (\textit{e.g.} motion, geometry, background information) 
to be integrated seamlessly. iv) We demonstrate that this robust framework achieves state-of-the-art on several cross-domain detection tasks.

\vspace{-2mm}
\section{Previous Work}
\vspace{-1mm}
\textbf{Object Detection:} The first approaches to object detection used a sliding window followed by a classifier based on hand-crafted features~\cite{dalal2005histograms, felzenszwalb2010object, viola2001rapid}. After advances in deep convolutional neural networks, methods such as R-CNN~\cite{girshick2014rich}, SPPNet~\cite{he2014spatial}, and Fast R-CNN~\cite{girshick2015fast} arose which used CNNs for feature extraction and classification. Slow sliding window algorithms were replaced with faster region proposal methods such as selective search ~\cite{uijlings2013selective}.
Recent object detection methods further speed bounding box detection. For example, in Faster R-CNN~\cite{ren2015faster} a region proposal network (RPN) was introduced to predict refinements in the locations and sizes of predefined anchor boxes. In SSD~\cite{liu2016ssd}, classification and bounding box prediction is performed on feature maps at different scales using anchor boxes with different aspect ratios. In YOLO~\cite{redmon2016you}, a regression problem on a grid is solved, where for each cell in the grid, the bounding box and the class label of the object centering at that cell is predicted. Newer extensions are found in~\cite{zhang2016faster, redmon2017yolo9000, dai2016r}. A comprehensive comparison of methods is reported in \cite{huang2017speed}. The goal of this paper is to increase the accuracy of an object detector in a new domain regardless of the speed. Consequently, we base our improvements on Faster R-CNN, a slower, but accurate detector.\footnote{Our adoption of faster R-CNN also allows for direct comparison with the state-of-the-art~\cite{chen2018domain}.}

\paragraph{Domain Adaptation:} was initially studied for image classification and the majority of the domain adaptation literature focuses on this problem~\cite{duan2012visual, duan2012domain, kulis2011you, gopalan2011domain, gong2012geodesic, fernando2013unsupervised, sun2016return, long2015learning, long2016unsupervised, ganin2016domain, ganin2014unsupervised, ghifary2016deep, busto2017open, motiian2017unified, li2018domain}. Some of the methods developed in this context include cross-domain kernel learning methods such as adaptive multiple kernel learning (A-MKL)~\cite{duan2012visual}, domain transfer multiple kernel learning (DTMKL)~\cite{duan2012domain}, and geodesic flow kernel (GFK)~\cite{gong2012geodesic}. There are a wide variety of approaches directed towards obtaining domain invariant predictors: supervised learning of non-linear transformations between domains using asymmetric metric learning ~\cite{kulis2011you}, unsupervised learning of intermediate representations~\cite{gopalan2011domain}, alignment of target and domain subspaces using eigenvector covariances~\cite{fernando2013unsupervised}, alignment the second-order statistics to minimize the shift between domains~\cite{sun2016return}, and covariance matrix alignment approach~\cite{wang2017deep}. The rise of deep learning brought with it steps towards domain-invariant feature learning. In \cite{long2015learning, long2016unsupervised} a reproducing kernel Hilbert embedding of the hidden features in the network is learned and mean-embedding matching is performed for both domain distributions. In \cite{ganin2016domain, ganin2014unsupervised} an adversarial loss along with a domain classifier is trained to learn features that are discriminative and domain invariant.

There is less work in domain adaptation for object detection. Domain adaptation methods for non-image classification tasks include \cite{gebru2017fine} for fine-grained recognition, \cite{chen2018road, hoffman2016fcns, zhang2017curriculum, vu2019advent} for semantic segmentation,  \cite{khodabandeh2018diy} for dataset generation, and \cite{mehrjou2018distribution} for finding out of distribution data in active learning. For object detection itself, \cite{xu2014domain} used an adaptive SVM to reduce the domain shift, \cite{raj2015subspace} performed subspace alignment on the features extracted from R-CNN, and \cite{chen2018domain} used Faster RCNN as baseline and took an adversarial approach (similar to \cite{ganin2014unsupervised}) to learn domain invariant features jointly on target and source domains. We take a fundamentally different approach by reformulating the problem as noisy labeling. We design a robust-to-noise training scheme for object detection which is trained on noisy bounding boxes and labels acquired from the target domain as pseudo-ground-truth.

\vspace{-3mm}
\paragraph{Noisy Labeling:}
Previous work on robust learning has focused on image classification where there are few and disjoint classes. Early work used instance-independent noise models, where each class is confused with other classes independent of the instance content~\cite{NatarajanNIPS13noisy, MnihICML12Arial, PatriniCVPR17, Sukhbaatar14Noisy, ZhangNIPS18generalized, yu2018learning}. Recently, the literature has shifted towards instance-specific label noise prediction~\cite{XiaoCVPR15, MisraCVPR16LabelingBias, VahdatNIPS17Robust, VahdatM13, VahdatZM14, VeitCVPR17Noisy, TanakaCVPR18Noisy, JiangICML18mentornet, DehghaniICLR18fidelity, RenICML18Robust}. To the best of our knowledge, ours is the first proposal for an object detection model that is robust to label noise.

\section{Method} \label{sec:method}

Following the common formulation for domain adaptation, we represent the training
data space as the source domain ($\S$) and the test data space as the target domain ($\T$).
We assume that an annotated training image dataset in $\S$ is supplied, but that only images in $\T$ are given (\textit{i.e.} there are no labels in $\T$).
Our framework, visualized in Fig.~\ref{fig:method}, consists of three main phases:

\begin{figure*}
\centering
\includegraphics[width=0.92\linewidth]{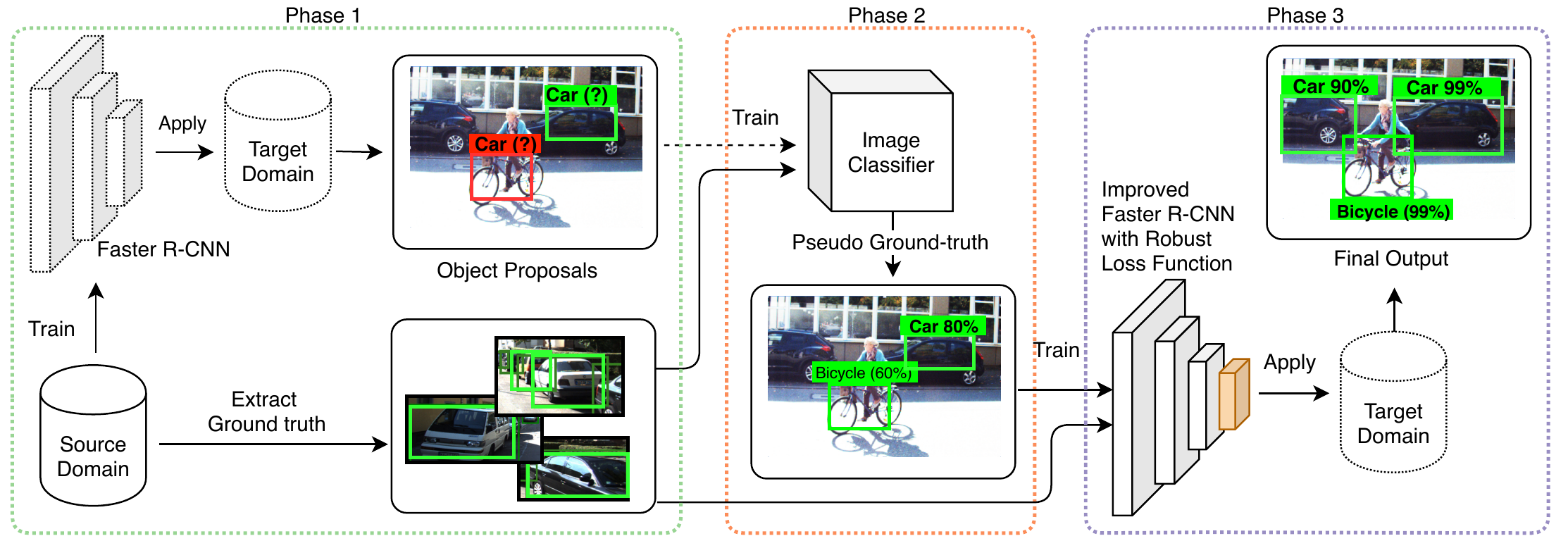}
\caption{The robust learning approach consists of three phases.
In phase 1, a detection module is trained using labeled data in the source domain.
This detector is then used to generate noisy annotations for images in the target domain. In phase 2, the annotations assigned in phase 1 are refined using a classification module. Finally, in phase 3, the detector is retrained using the original labeled data and the refined machine-generated annotations in the target domain. Retraining is formulated to account for the possibility of mislabeling.}
\label{fig:method}
\vspace{-1mm}
\end{figure*}
\begin{enumerate}
\vspace{-2mm}
\item \textbf{Object proposal mining}: A standard Faster R-CNN, trained
on the source domain, is used to detect objects in the target domain.
The detected objects form a proposal set in $\T$.
\vspace{-2mm}
\item \textbf{Image classification training}: Given the images extracted
from bounding boxes in $\S$, we train an image classification model
that predicts the class of objects in each image. The resulting
classifier is used to score the proposed bounding boxes in $\T$.
This model aids in training the robust object detection model in 
the next phase. The reason for introducing image classification is that 
i) this model may rely on representations
different than those used by the phase one detection model (\textit{e.g.}, motion features)
or it may use a more sophisticated network architectures, 
and ii) this model can be trained in a semi-supervised fashion using
labeled images in $\S$ and unlabeled images in $\T$.
\vspace{-2mm}
\item \textbf{~Robust object detection training}: 
In this phase a robust object detection model is trained using
object bounding boxes in $\S$ and object proposals in $\T$ (from phase one)
that has been rescored using the image classification (from phase two).
\end{enumerate}

We organize the detailed method description as follows. Firstly, we introduce background notation and provide a description of Faster R-CNN in Sec.~\ref{sec:faster_rcnn} to define the model used in phase one. Secondly, a probabilistic view of
Faster R-CNN in Sec.~\ref{sec:prob_faster_rcnn} provides
a foundation for the robust object detection framework presented in Sec.~\ref{sec:robust_det}. This defines the model used in phase three. Lastly, the image classification model used in phase two is discussed in Sec.~\ref{sec:image_classification}.

\vspace{-2mm}
\paragraph{Notation:} We are given training images in $\S$ along with
their object bounding box labels. This training set is denoted by $\D_{\S} = \{(\x^{(s)}, \y^{(s)})\}$ 
where $\x^{(s)} \in \S$ represents an image, 
$\y^{(s)}$ is the corresponding bounding box label for $\x^{(s)}$ 
and $s$ is an index. Each bounding box $\y=(y_c, \y_l)$
represents a class label by an integer, $y_c \in \Y=\{1, 2, \dots, C\}$, where $C$ is the number
of foreground classes, and a 4-tuple, $\y_l \in \mathcal{R}^4$, giving the coordinates of the top left corner, height, and width of the box. To simplify
notation, we associate each image with a single bounding box.\footnote{This restriction is for notational convenience only. Our implementation makes no assumptions
about the number of objects in each image.}

In the target domain, images are given without bounding box annotations. 
At the end of phase one, we augment this dataset with proposed bounding boxes generated by
Faster R-CNN. We denote the resulting set by $\D_{\T} = \{\x^{(t)}, \yt^{(t)}\}$ where 
$\x^{(t)} \in \T$ is an image, $\yt^{(t)} \in \Y$ is the corresponding proposed
bounding box
and $t$ is an index. Finally, we obtain the image classification score obtained at the end of phase two
for each instance in $\D_{\T}$ from $\pimg(y_c| \x, \yt_l)$ which
represents the probability of assigning the image cropped in the bounding box $\yt_l$
in $\x$ to the class $y_c \in \Y \cup \{0\}$ which is one of the foreground categories or background.

\subsection{Faster R-CNN} \label{sec:faster_rcnn}
Faster R-CNN~\cite{ren2015faster} is a two-stage detector consisting of two main components:
a region proposal network (RPN) that proposes regions of interests (ROI) for object detection
and an ROI classifier that predicts object labels for the proposed bounding boxes. 
These two components share the first convolutional layers. Given an input image, the shared layers extract a feature map for the image.
In the first stage, RPN predicts the probability of a set of predefined anchor boxes for being an object or background along with refinements 
in their sizes and locations. The anchor boxes are a fixed predefined set of boxes with varying positions, sizes and aspect ratios across the image.
Similar to RPN, the region classifier predicts object labels for ROIs proposed by the RPN as well as refinements for the location and size of the boxes. 
Features passed to the classifier are obtained with a \textit{ROI-pooling} layer. Both networks are trained jointly by minimizing a loss function:
\vspace{-2mm}
\begin{equation}\label{eq:faster_rcnn}
    \LL = \LL_{RPN} + \LL_{ROI}.
\vspace{-2mm}
\end{equation}

$\LL_{RPN}$ and $\LL_{ROI}$ represent losses used for the RPN and ROI classifier. The losses consist of a cross-entropy cost measuring
the mis-classification error and a regression loss quantifying the localization error.
The RPN is trained to detect and localize objects without regard to
their classes, and the ROI classification network is trained to classify the object labels. 

\subsection{A Probabilistic View of Faster R-CNN} \label{sec:prob_faster_rcnn}
In this section, we provide a probabilistic view of Faster R-CNN that will be used to define a robust loss function for noisy detection labels.
The ROI classifier in Faster R-CNN generates an object
classification score and object location for each proposed bounding box generated
by the RPN.
A classification prediction $\pcls(y_c| \x, \yt_l)$ represents the probability of a categorical random variable taking one of the disjoint $C + 1$ classes (\textit{i.e.}, foreground classes plus background).
This classification distribution is modeled using a softmax activation.
Similarly, we model the location prediction
$\ploc(\y_l| \x, \yt_l) = \N(\y_l; \yh_l, \sigma \eye)$ with a multivariate Normal distribution\footnote{This assumption follows naturally if the L$_2$-norm is used for the localization error in Eq.~\ref{eq:faster_rcnn}. 
In practice however, a combination of L$_2$ and L$_1$ norms are used which do not correspond to a simple probabilistic output.} with mean $\yh_l$ and constant diagonal covariance matrix $\sigma \eye$. 
In practice, only $\yh_l$ is generated by the ROI classifier which is used to localize the object.

\subsection{Robust Faster R-CNN} \label{sec:robust_det}
\vspace{-2mm}
To gain robustness against detection noise on both the label ($y_c$) and the box location/size ($\y_l$), 
we develop a refinement mechanism that corrects mistakes in both class and box location/size annotations. The phase three detection model is trained using these
refined annotations. 

If the training annotations are assumed to be noise-free then both $\pcls$ and $\ploc$ are used to define the maximum-likelihood loss
functions in Eq.~\ref{eq:faster_rcnn}. In the presence of noisy labels,
$\argmax \pcls$ and $ \argmax\ploc$ may disagree with the noisy labels but nevertheless correctly identify the true class or location of an object. 
Additionally, we also have access to the image classification model $\pimg$ from phase 2 that may be more accurate in predicting class labels for proposed bounding boxes in $\T$ since it
is trained using information sources different from the primary detection model.
The question then is how to combine $\pcls$, $\ploc$ from Faster R-CNN and $\pimg$ from the image model
to get the best prediction for the class and location of an object?

Vahdat~\cite{VahdatNIPS17Robust} has proposed a regularized EM algorithm for robust training
of image classification models. Inspired by this approach, we develop two mechanisms for correcting classification and localization errors, based on the assumption that when training a classification model 
on noisy labeled instances, the distribution over true labels should be close to both the distributions generated by the underlying classification model
and an auxiliary distribution obtained from other sources. Since the accuracy
of the learned classification model improves during training, the weighting of these information sources should shift during training. 
\vspace{-4mm}

%Vahdat~\cite{VahdatNIPS17Robust} has proposed a regularized EM algorithm for robust training
%of image classification models. That approach is based on the assumption that when training an image classification model 
%on noisy labeled instances, the distribution over true labels should be close to both the distributions generated by the underlying classification model
%and an auxiliary distribution obtained from other sources of information. Since the accuracy
%of the learned classification model improves during training, the weighting of these information sources should shift during training. 
%Inspired by \cite{VahdatNIPS17Robust}, we develop two mechanisms 
%for correcting classification and localization errors.

\paragraph{Classification Error Correction:}
We seek a distribution, $q(y_c)$, which is close to both the classification model of Faster R-CNN and the image classification model $\pimg$, that is trained in phase two. We propose the following optimization objective for inferring $q(y_c)$
\begin{equation}\label{eq:kl_classification} \small \hspace{-0.2cm}
\min_{q} \KL(q(y_c) || \pcls(y_c| \x, \yt_l))\!+\!\alpha \KL(q(y_c) || \pimg(y_c| \x, \yt_l)).
\end{equation}
$\KL$ denotes the Kullback-Leibler divergence and $\alpha>0$ balances the trade-off between two terms. 
With large values of $\alpha$, $q$ favors the image classification model ($\pimg$) over Faster R-CNN predictions ($\pcls$), and with smaller $\alpha$, $q$ 
favors $\pcls$. Over the course of training, $\alpha$ can be changed to set a reasonable balance between the two distributions.

The following result provides a closed-form solution to the optimization
problem in Eq.~\ref{eq:kl_classification}:
\begin{theorem} \label{th:theorem1} Given two probability distributions $p_1(z)$ and $p_2(z)$
defined for the random variable $z$ and positive scalar $\alpha$, the
closed-form minimizer of
\begin{equation*}
    \min_{q} \KL(q(z) || p_1(z))+\alpha \KL(q(z) || p_2(z))
\vspace{-3mm}
\end{equation*}
\vspace{-5mm}
is given by:
\begin{equation} \label{eq:q_solution}
    q(z) \propto \big(p_1(z)p_2^\alpha(z) \big)^{\frac{1}{\alpha+1}}
\end{equation}
\end{theorem}
\begin{proof} Here, we prove the theorem for a continuous random variable defined in domain $\Omega$.
\begin{eqnarray} \hspace{-0.1cm}
\min_{q} && \KL(q(z) || p_1(z))+\alpha \KL(q(z) || p_2(z)) \nonumber \\
&& = \int_{\Omega} q(z) \log \frac{q(z)}{p_1(z)} dz+\alpha \int_{\Omega} q(z) \log \frac{q(z)}{p_2(z)} dz \nonumber \\
&& = (\alpha + 1) \int_{\Omega} q(z) \log \frac{q(z)}{\big[p_1(z)p_2^\alpha(z)\big]^{\frac{1}{\alpha + 1}}} dz \nonumber \\
&& = (\alpha + 1) \KL(q(z) \ || \ \frac{1}{Z}\big[p_1(z)p_2^\alpha(z)\big]^{\frac{1}{\alpha + 1}}) + C \nonumber
\end{eqnarray}
where $Z$ is the normalization for $\big(p_1(z)p_2^\alpha(z) \big)^{\frac{1}{\alpha+1}}$ and $C$ is a constant independent of $q$. The final KL is minimized when Eq.~\ref{eq:q_solution} holds.
\end{proof}
Using Theorem.~\ref{th:theorem1}, the solution to Eq.~\ref{eq:kl_classification}
is obtained as the weighted geometric mean of the two distributions:
\begin{equation}\label{eq:q_classification}
q(y_c) \propto \big(\pcls(y_c| \x, \yt_l)\pimg^\alpha(y_c| \x, \yt_l) \big)^{\frac{1}{\alpha+1}}.
\end{equation}
Since both $\pcls(y_c| \x, \yt_l)$ and $\pimg(y_c| \x, \yt_l)$ are categorical distributions (with softmax activation), $q(y_c)$ is
also a (softmax) categorical distribution whose parameters are 
obtained as the weighted mean of the
logits generated by $\pcls$ and $\pimg$, \textit{i.e.}, $\sigma\bigl((\ll_{cls} + \alpha \ll_{img})/(1+\alpha)\bigr)$ where $\sigma$ is the softmax and $\ll_{cls}$ and $\ll_{img}$ are the corresponding logits. Setting $\alpha = \infty$ in Eq.~\ref{eq:q_classification}
sets $q(y_c) $ to $\pimg(y_c| \x, \yt_l)$ while $\alpha = 0$
sets $q(y_c) $ to $\pcls(y_c| \x, \yt_l)$. 
During training we reduce $\alpha$ from large to smaller values. Intuitively, at the beginning of the training, $\pcls(y_c| \x, \yt_l)$ 
is inaccurate and provides a poor estimation of the true class labels, therefore by setting $\alpha$ to a large value we guide $q(y_c)$ to rely on $\pimg(y_c| \x, \yt_l)$ 
more than $\pcls$. By decreasing $\alpha$ throughout training, $q$ will rely on both $\pcls$ and $\pimg$ to form a distribution over true class labels. 
\vspace{-3mm}

\paragraph{Bounding Box Refinement:} 
Eq.~\ref{eq:q_classification} refines the classification labels
for the proposal bounding boxes in the target domain.
Here, we provide a similar method for correcting the errors in location and size.
Recall that Faster R-CNN's location predictions for the proposal bounding boxes
can be thought as a Normally distributed $\N(\y_l; \yt_l, \sigma \eye)$ with mean $\yt_l$ and constant diagonal covariance matrix $\sigma \eye$. 
We let $\pinit(\y_l|\x, \yt_l) := \N(\y_l; \yt_l, \sigma \eye)$
denote the initial detection for image $\x$.
At each iteration Faster R-CNN predicts a location for object using $\ploc(\y_l| \x, \yt_l) = \N(\y_l; \yh_l, \sigma \eye)$ for image $\x$
and the proposal $\yt_l$. We use the following objective function for inferring a distribution $q$ over true object locations:
\begin{equation}\label{eq:kl_location} \small \hspace{-0.2cm}
\min_{q} \KL(q(\y_l) || \ploc(\y_l| \x, \yt_l))\!+\!\alpha \KL(q(\y_l) || \pinit(\y_l|\x, \yt_l))
\end{equation}
As with Eq.~\ref{eq:kl_classification}, the solution to
Eq.~\ref{eq:kl_location} is the weighted geometric mean of the two distributions.
\begin{theorem} \label{th:theorem2}
Given two multivariate Normal distributions $p_1(\z) = \N(\z; \bmu_1, \bSigma)$ and $p_2(\z) = \N(\z; \bmu_2, \bSigma)$ with common covariance matrix $\bSigma$
defined for the random variable $\z$ and a positive scalar $\alpha$, the weighted
geometric mean $q(\z) \propto \big(p_1(\z)p_2^\alpha(\z) \big)^{\frac{1}{\alpha+1}}$ is also Normal with mean
$\big(\bmu_1 + \alpha \bmu_2\big)/(\alpha + 1)$ and covariance matrix $\bSigma$.
\end{theorem}
\begin{proof} By the definition of the Normal distribution, we have:
\begin{align}
q(\z) & \propto \big(p_1(\z)p_2^\alpha(\z) \big)^{\frac{1}{\alpha+1}} \nonumber \\
& \propto e^{-\frac{1}{2}\big[ \frac{1}{\alpha+1}(\z - \bmu_1)^T \bSigma^{-1}(\z - \bmu_1) + \frac{\alpha}{\alpha+1}(\z - \bmu_2)^T \bSigma^{-1}(\z - \bmu_2) \big]} \nonumber \\
&\propto e^{-\frac{1}{2}\big[ \z^T \bSigma^{-1} \z - 2\z^T \bSigma^{-1} (\frac{\bmu_1 + \alpha \bmu_2}{\alpha+1}) \big]}  \nonumber \\
&\propto e^{-\frac{1}{2} \big(\z - (\frac{\bmu_1 + \alpha \bmu_2}{\alpha+1}) \big)^T \bSigma^{-1}\big(\z - (\frac{\bmu_1 + \alpha \bmu_2}{\alpha+1})\big) }  \nonumber
\end{align}
Hence, $q(\z) = \N(\z; (\bmu_1 + \alpha \bmu_2)/(\alpha + 1), \bSigma)$
\end{proof}

Using Theorem.~\ref{th:theorem2}, the minimizer of Eq.~\ref{eq:kl_location}
is:
\begin{equation} \label{eq:q_loc_solution}
    q(\y_l) = \N\bigl(\y_l; (\yh_l + \alpha \yt_l)/(\alpha+1), \sigma \eye\bigr).
\end{equation}
This result gives the refined bounding
box location and size as the weighted average of box location/size
extracted from phase one and the current output of Faster R-CNN. Setting $\alpha=\infty$ ignores the current output of Faster R-CNN while $\alpha=0$ uses its output as the location. At training time, we initially set $\alpha$ to a large value and then gradually decrease it to smaller values. In this way, at
early stages of training $q$ relies on $\pinit$ because it's more accurate than the current estimation of the model, but as training progresses and $\ploc$ becomes more accurate,
$q$ relies more heavily on $\ploc$.

\paragraph{Training Objective Function:} We train a robust Faster R-CNN
using $\D_\S \cup \D_\T$. At each minibatch update, if an instance belongs 
to $\D_\S$ then the original loss function of Faster R-CNN is used for parameter update.
If an instance belongs to $\D_\T$ then $q(y_c)$ in Eq.~\ref{eq:q_classification} and $q(\y_l)$ in Eq.~\ref{eq:q_loc_solution} are used to refine the proposed
bounding box annotations.
$q(y_c)$ is used as the soft target labels in the cross entropy loss function
for the mis-classification term and $(\yh_l + \alpha \yt_l)/(\alpha+1)$
is used as the target location for the regression term. The modifications
are made only in the ROI classifier loss function because the RPN is class agnostic.

\paragraph{False Negative Correction:} Thus far, the robust detection method only refines the object
proposals generated in phase one. This allows the model to correct false positive detections, \textit{i.e.},
instances that do not contain any foreground object or that contain
an object from a class different than the predicted class. However, we would also like to correct false negative predictions, 
\textit{i.e.}, positive instances of foreground classes that are not detected in phase one.

To correct false negative instances, we rely on the hard negative
mining phase of Faster R-CNN. In this phase a set of hard negative
instances are added as background instances to the training set. Hard negatives 
that come from $\D_\S$ are actually background images. However, the ``background'' instances that are extracted from $\D_\T$ may 
be false negatives of phase one and may contain foreground objects. Therefore, 
during training for negative samples that belong to $\D_\T$, we define $\pimg(y_c)$ to be a softened one-hot vector by setting the probability of a background to $1-\epsilon$ and the probability of the other class labels uniformly to $\epsilon / C$. 
This is used as a soft target label in the cross-entropy loss.

\subsection{Image Classification:} \label{sec:image_classification}

Phase two of our framework uses an image
classification model to re-score bounding box proposals obtained
in phase one. %We use a classifier with InceptionV4~\cite{szegedy2017inception} architecture which is known to be stronger than Faster R-CNN's classification module.  This provides our model with better class label scores on the proposed bounding boxes. 
The image classification network is trained in a semi-supervised setting on top of images cropped from both $\D_\S$ (clean training set) and $\D_\T$ (noisy labeled set). For images in $\D_\S$, we use the cross-entropy loss against ground truth labels, but, for images in $\D_\T$ the cross-entropy loss is computed against soft labels obtained by Eq~\ref{eq:kl_classification}, where the weighted geometric mean between predicted classification score and a softened one-hot annotation vector is computed. This corresponds to multiclass extension of~\cite{VahdatNIPS17Robust} which allows the classification model to refine noisy class labels for images in $\D_\T$.

Note that both $\D_\S$ and $\D_\T$ have bounding boxes annotations from foreground
classes (although instances in $\D_\T$ have noisy labels).
For training the image classification models, we augment these two datasets
with bounding boxes mined from areas in the image that do not have overlap with
bounding boxes in $\D_\S$ or $\D_\T$.

\section{Experiments}
To compare with state-of-the-art methods we follow the experimental design of \cite{chen2018domain}. We perform three experiments on three source/target domains and use similar hyper-parameters as \cite{chen2018domain}. %Rather than the original Caffe implementation of Faster R-CNN
We use the Faster R-CNN implementation available in the object detection API~\cite{huang2017speed} source code. In all the experiments, including the baselines and our method, we set the initial learning rate to $0.001$ for $50,000$ iterations and reduce it to $0.0001$ for the next $20,000$ iterations (a similar training scheme as \cite{chen2018domain}). We linearly anneal $\alpha$ from $100$ to $0.5$ for the first $50,000$ iterations and keep it constant thereafter. We use InceptionV2~\cite{szegedy2016rethinking}, pre-trained on ImageNet~\cite{deng2009imagenet}, as the backbone for Faster R-CNN.  In one slight departure, we set aside a small portion of the training set as validation for setting hyper-parameters. InceptionV4~\cite{szegedy2017inception} is used for the image classification phase with initial learning rate of $3\times10^{-4}$ that drops every $2$ epochs by a factor $0.94$. We set the batch size to $32$ and train for $300\,000$ steps.

\paragraph{Baselines:} We compare our method against the following progressively more sophisticated baselines.
\begin{itemize}
\setlength\itemsep{0.1em}
    \item \textbf{Faster R-CNN}~\cite{ren2015faster}: This is the most primitive baseline. A Faster R-CNN object detector is trained on the source domain and tested on the target domain so that the object detector is blind to the target domain. 
    
    \item \textbf{Pseudo-labeling}~\cite{inoue2018cross}: A simplified version of our method in which Faster R-CNN is trained on the source domain to extract object proposals in the target domain, and then based on a pre-determined threshold, a subset of the object proposals are selected and used for fine-tuning Faster R-CNN. This process can be repeated.  This method corresponds to the special case where $\alpha=0$ is fixed throughout training.
    The original method in \cite{inoue2018cross} performs a progressive adaptation, which is computationally extensive. Since our method and the previous state-of-the-art method perform only one extra fine-tuning step, we perform only one repetition for a fair comparison.
    
    \item \textbf{Feature Learning}~\cite{chen2018domain}: This state-of-the-art domain adaptation method reduces the domain discrepancy by learning robust features in an adversarial manner. We follow the experimental setup used in \cite{chen2018domain}.
    
%    \item \textbf{Noisy Labeling (Our method)}. In a nutshell, object proposals are extracted, similar to Pseudo-labeling. However, class labels depends on both current estimation of Faster R-CNN and estimation of an auxiliary source of information (a strong classifier). The exact formula and the loss that we used is explained in Sec.~\ref{sec:method}. Additionally, our method refines the object bounding boxes, which other methods are not designed to explicitly do such refinements.

\vspace{-3mm}
\end{itemize}

\paragraph{Datasets:} Following \cite{chen2018domain} we evaluate performance on multi- and single-label object detection tasks using three different datasets. Depending on the experiment, some datasets are used as both target and source domains and some are only used as either the source or target domain. 
\begingroup
\setlength{\tabcolsep}{4.1pt}
\begin{table*}[!ht]
\renewcommand\thetable{2}
\small
\begin{center}
\begin{tabularx}{\linewidth}{l | c c c |c c c c c c c c | c}
\hline
\multicolumn{13}{c}{\city{} $\rightarrow$ \cityfog{}} \\
\bottomrule
Method & Cls-Cor & Box-R & FN-Cor & person & rider & car & truck & bus & train & motorcycle & bicycle & mAP \\
\hline\hline
\multicolumn{2}{l}{\faster} &  & & 31.69 & 39.41 & 45.81 & 23.86 & 39.34 & 20.64 & 22.26 & 32.36 & 31.92 \\ \hline
\multicolumn{2}{l}{\mani} & & & 31.94 & 39.94 & 47.97 & 25.13 & 39.85 & 27.22 & 25.01 & 34.12 & 33.90  \\ \hline
\multicolumn{2}{l}{\adaptive} & & & \textbf{35.81} & 41.63 & 47.36 & 28.49 & 32.41 &  \textbf{31.18} & 26.53 & 34.26 & 34.70 \\ \hline
\multirow{3}{*}{\ours:} & \cmark & \xmark & \xmark & 34.82 & 41.89 & 48.93 & 27.68 & 42.53 & 26.72 & 26.65 & 35.76 & 35.62  \\ 
& \cmark & \cmark & \xmark & 35.26 & \textbf{42.86} & \textbf{50.29} & 27.87 & 42.98 & 25.43 & 25.30 & 35.94 & 36.06 \\
& \cmark & \cmark & \cmark & 35.10 & 42.15 & 49.17 & \textbf{30.07} & \textbf{45.25} & 26.97 & \textbf{26.85} & \textbf{36.03} & \textbf{36.45 } \\
\hline\hline
\multicolumn{2}{l}{\faster trained on target} & & & 40.63 & 47.05 & 62.50 & 33.12 & 50.43 & 39.44 & 32.57 & 42.43 & 43.52 \\ 
\hline
\end{tabularx}
\end{center}
\caption{Quantitative results comparing our method to baselines for adapting from \city{} to \cityfog{}. We record the average precision (AP) on the \city{} validation set. ``Cls-Cor'' represents ``classification error correction'', Box-R stands for ``Bounding Box Refinement'' component, and FN-Cor stands for ``False Negative Correction'' component of our method. The last row shows the base detector's performance if labeled data for target domain was available. }
\label{tab:city2fog}
\end{table*}
\endgroup
\begin{itemize}
\vspace{-4mm}
\setlength\itemsep{0.1em}
    \item \textbf{\simk{}}~\cite{johnson2016driving} is a simulated dataset containing $10,000$ images synthesized by the Grand Theft Auto game engine. In this dataset, which simulates car driving scenes captured by a dash-cam, there are $58,701$ annotated car instances with bounding boxes. We use $10\%$ of these for validation and the remainder for training. 
    
    \item \textbf{\city{}}~\cite{cordts2016cityscapes} is a dataset\footnote{This dataset is usually used for instance segmentation and not object detection.} of real urban scenes containing $3,475$ images captured by a dash-cam, $2,975$ images are used for training and the remaining $500$ for validation. Following ~\cite{chen2018domain} we report results on the validation set because the test set doesn't have annotations.  In our experiments we used the tightest bounding box of an instance segmentation mask as ground truth. There are $8$ different object categories in this dataset including \textit{person, rider, car, truck, bus, train, motorcycle and bicycle}. 
    
    \item \textbf{\cityfog{}}~\cite{sakaridis2018semantic} is the foggy version of \city{}. The depth maps provided in \city{} are used to simulate three intensity levels of fog in \cite{sakaridis2018semantic}. In our experiments we used the fog level with highest intensity (least visibility). The same dataset split used for \city{} is used for \cityfog{}. 
    
    \item \textbf{\kit{}}~\cite{Geiger2013IJRR} is another real-world dataset consisting of $7,481$ images of real-world traffic situations, including freeways, urban and rural areas. Following~\cite{chen2018domain} we used the whole dataset for both training, when it is used as source, and test, when it is used as target.
\end{itemize}

\subsection{Adapting synthetic data to real world}
In this experiment, the detector is trained on synthetic data generated using computer simulations and the model is adapted to real world examples. This is an important use case
as it circumvents the lack of annotated training data common to many applications (\textit{e.g}. autonomous driving). The source domain is \simk{} and the target domain 
is \city{} dataset (denoted by ``\simk{} $\rightarrow$ \city{}''). We use the validation set of \city{} for evaluating the results. We only train the detector on annotated \textit{cars}
because \textit{cars} is the only object common to both \simk{} and \city{}. 

\begin{table}[!htbp]
\renewcommand\thetable{1}
\setlength{\tabcolsep}{3.2pt}
\begin{center}
\small
\begin{tabular}{l | c c c |c}
\hline
\multicolumn{5}{c}{\simk{} $\rightarrow$ \city{}} \\
\bottomrule
Method  & Cls-Cor & Box-R & FN-Cor & AP \\
\hline\hline
\multicolumn{2}{l}{\faster} & & & 31.08 \\ \hline
\multicolumn{2}{l}{\mani} & & & 39.05  \\ \hline
\multicolumn{2}{l}{\adaptive} & & &  40.10 \\ \hline
\multirow{3}{*}{\ours:} & \cmark & \xmark & \xmark &  \textbf{41.28}  \\ 
& \cmark & \cmark & \xmark & \textbf{41.83} \\
& \cmark & \cmark & \cmark & \textbf{42.56} \\
\hline\hline
\multicolumn{2}{l}{\faster trained on target} & & & 68.10 \\ 
\hline
\end{tabular}
\end{center}
\caption{ Quantitative results comparing our method to baselines for adapting from \simk{} dataset to \city{}. We record average precision (AP) on the \city{} validation set. The last row shows the base detector's performance if labeled data for target domain was available.}
\label{tab:sim2city}
\end{table}

Table~\ref{tab:sim2city} compares our method to the baselines. We tested our method with ``Classification Error Correction (Cls-Cor)''\footnote{Turning off Cls-Cor reduces our approach to a method similar to \mani with similar performance. To maintain robustness to label noise, we run all experiments with Cls-Cor component.}, with or without the ``Bounding Box Refinement (Box-R)'' and ``False Negative Correction (FN-Cor)'' components. The state-of-the-art \adaptive method has $+1.05\%$ improvement over the basic \mani baseline. Our best performing method has a $+3.51\%$ improvement over the same baseline yielding more than triple the improvement over the incumbent state-of-the-art.

\begin{figure*} [!ht]
\centering
\renewcommand*{\arraystretch}{0.5}
\begin{tabular}{>{\centering\arraybackslash}m{0.004\linewidth}*{6}{m{0.188\linewidth}@{\hskip 3pt}}@{}}
 \rotatebox[origin=c]{90}{Faster R-CNN}   &
 \xintFor #1 in {1,2,4,5,6} \do{
     \includegraphics[width=\linewidth]{#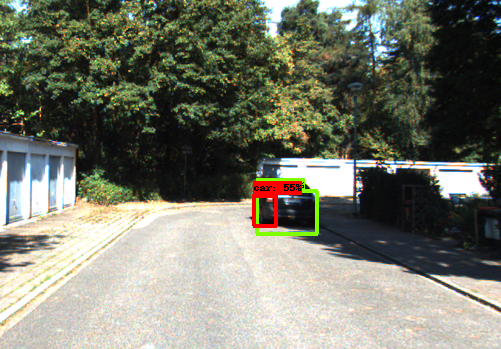} & } \\
 \rotatebox[origin=c]{90}{Ours} &
 \xintFor #1 in {1,2,4,5,6} \do{
     \includegraphics[width=\linewidth]{#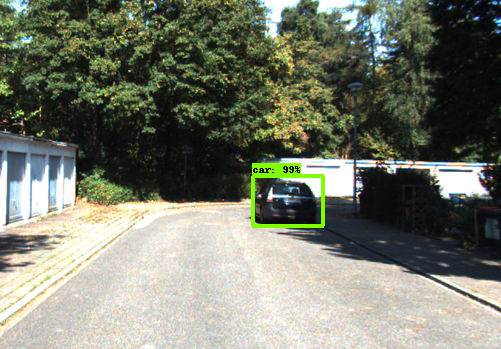} & } 
\end{tabular}

\caption{Qualitative comparison of our method with Faster R-CNN on the ``\city{} $\rightarrow$ \kit{}'' experiment. Each column corresponds to a particular image in the \kit{} test set. Top and bottom images in each column illustrate the bounding boxes of the cars detected by Faster R-CNN and our method respectively. In the first two columns our method corrects several false positives. In all cases our method successfully corrected the size/location of the bounding boxes (\textit{e.g.} the rooflines in the third column). In the fourth and fifth examples, our method has detected cars that Faster R-CNN has missed. Nevertheless, false positives do occur (\textit{e.g.} in column five), though the probability of those specific false positives is low ($53\%$ in this example).}
\label{fig:qualitative}
\end{figure*}

\subsection{Adapting normal to foggy weather}

Changes in weather conditions can significantly affect visual data. In applications such as autonomous driving, the object detector must perform accurately in all conditions~\cite{sakaridis2018semantic}. However, it is often not possible to capture all possible variations of objects in all weather conditions. Therefore, models must be adaptable to differing weather conditions. Here we evaluate our method and demonstrate its superiority over the current state-of-the-art for this task. We use \city{} dataset as the source domain and \cityfog{} as the target domain (denoted by ``\city{} $\rightarrow$ \cityfog{}''). 

Table.~\ref{tab:city2fog} compares our method to the baselines on multi-label domain adaptation. The categories in this experiment are \textit{person, rider, car, truck, bus, train, motorcycle, bicycle}. Average precision for each category along with the mean average precision (mAP) of all the objects are reported. Our method improves Faster R-CNN mAP by $+4.53\%$, while the state-of-the-art's improvement is $+2.78\%$.

\subsection{Adapting to a new dataset}
The previous examples of domain adaptation (synthetic data and weather change) are somewhat specialized. However, any change 
in camera (\textit{e.g.} angle, resolution, quality, type, etc.) or environmental setup can cause domain shift.
We investigate the ability of our method to adapt from one real dataset to another real dataset. We use \city{} and \kit{} as the source and target domain in two 
separate evaluations. We denote the experiment in which \city{} is the source domain and \kit{} is the target domain by ``\city{}~$\rightarrow$~\kit{}'', and vice versa by ``\kit{}~$\rightarrow$~\city{}''.

Tables~\ref{tab:citytokitti} and \ref{tab:kittitocity} compare average precision on the \textit{car} class, the only common object. Our method significantly outperforms the state-of-the-art in both situations (\city{} $\rightleftarrows$ \kit{}). Qualitative 
results of our method on the \kit{} test set are shown in Figure~\ref{fig:qualitative}.

\section{Conclusion}
Domain shift can severely limit the real-world deployment of object-detection-based applications when labeled data collection is either expensive or infeasible. 
We have proposed an unsupervised approach to mitigate this problem by formulating the problem as robust learning. 
Our robust object detection framework copes with labeling noise on both object classes and bounding boxes. State-of-the-art performance is achieved by 
robust training in the target domain using a model trained only in the source domain. This approach eliminates the need for collecting data in the target 
domain and integrates other sources of information using detection re-scoring.

\section{Acknowledgements}
This work was supported by the Security R\&D Group of SK Telecom.

\begin{table}[!htbp]
\setlength{\tabcolsep}{3.2pt}
\begin{center}
\small
\begin{tabular}{l | c c c |c }
\hline
\multicolumn{5}{c}{\kit{} $\rightarrow$ \city{}} \\
\bottomrule
Method & Cls-Cor & Box-R & FN-Cor & AP \\
\hline\hline
\multicolumn{2}{l}{\faster} & & & 31.10 \\ \hline
\multicolumn{2}{l}{\mani} & & & 40.23  \\ \hline
\multicolumn{2}{l}{\adaptive} & & &  40.57 \\ \hline
\multirow{3}{*}{\ours:} & \cmark & \xmark & \xmark &  \textbf{42.03}  \\ 
& \cmark & \cmark & \xmark & \textbf{42.39} \\
& \cmark & \cmark & \cmark & \textbf{42.98} \\
\hline\hline
\multicolumn{2}{l}{\faster trained on target} & & & 68.10 \\
\hline
\end{tabular}
\end{center}
\caption{Quantitative comparison of our method with baselines for adapting from \kit{} to \city{}. We record average precision (AP) on the \city{} test set. The last row gives the base detector's performance if labeled data for the target domain was available.}
\label{tab:citytokitti}
\end{table}

\begin{table}[!htbp]
\setlength{\tabcolsep}{3.2pt}
\begin{center}
\small
\begin{tabular}{l | c c c |c }
\hline
\multicolumn{5}{c}{\city{} $\rightarrow$ \kit{}} \\
\bottomrule

Method & Cls-Cor & Box-R & FN-Cor & AP \\
\hline\hline
\multicolumn{2}{l}{\faster} & & & 56.21 \\ \hline
\multicolumn{2}{l}{\mani} & & & 73.84  \\ \hline
\multicolumn{2}{l}{\adaptive} & & &  73.76 \\ \hline
\multirow{3}{*}{\ours:} & \cmark & \xmark & \xmark &  \textbf{76.36}  \\ 
& \cmark & \cmark & \xmark & \textbf{76.93} \\
& \cmark & \cmark & \cmark & \textbf{77.61} \\
\hline\hline
\multicolumn{2}{l}{\faster trained on target} & & & 90.13 \\
\hline
\end{tabular}
\end{center}
\caption{Quantitative comparison of our method with baselines for adapting \city{} to \kit{}. We record average precision (AP) on the \kit{} train set. The last row gives the base detector's performance if labeled data for target domain was available.}
\label{tab:kittitocity}
\end{table}

{\small
\bibliographystyle{ieee_fullname}
\bibliography{egbib}

\begin{thebibliography}{10}\itemsep=-1pt

\bibitem{busto2017open}
P.~P. Busto and J.~Gall.
\newblock Open set domain adaptation.
\newblock In {\em ICCV}, pages 754--763, 2017.

\bibitem{chen2018domain}
Y.~Chen, W.~Li, C.~Sakaridis, D.~Dai, and L.~Van~Gool.
\newblock Domain adaptive faster r-cnn for object detection in the wild.
\newblock In {\em Proceedings of the IEEE Conference on Computer Vision and
  Pattern Recognition}, pages 3339--3348, 2018.

\bibitem{chen2018road}
Y.~Chen, W.~Li, and L.~Van~Gool.
\newblock Road: Reality oriented adaptation for semantic segmentation of urban
  scenes.
\newblock In {\em Proceedings of the IEEE Conference on Computer Vision and
  Pattern Recognition}, pages 7892--7901, 2018.

\bibitem{cordts2016cityscapes}
M.~Cordts, M.~Omran, S.~Ramos, T.~Rehfeld, M.~Enzweiler, R.~Benenson,
  U.~Franke, S.~Roth, and B.~Schiele.
\newblock The cityscapes dataset for semantic urban scene understanding.
\newblock In {\em Proceedings of the IEEE conference on computer vision and
  pattern recognition}, pages 3213--3223, 2016.

\bibitem{dai2016r}
J.~Dai, Y.~Li, K.~He, and J.~Sun.
\newblock R-fcn: Object detection via region-based fully convolutional
  networks.
\newblock In {\em Advances in neural information processing systems}, pages
  379--387, 2016.

\bibitem{dalal2005histograms}
N.~Dalal and B.~Triggs.
\newblock Histograms of oriented gradients for human detection.
\newblock In {\em Computer Vision and Pattern Recognition, 2005. CVPR 2005.
  IEEE Computer Society Conference on}, volume~1, pages 886--893. IEEE, 2005.

\bibitem{DehghaniICLR18fidelity}
M.~Dehghani, A.~Mehrjou, S.~Gouws, J.~Kamps, and B.~Sch{\"o}lkopf.
\newblock Fidelity-weighted learning.
\newblock In {\em International Conference on Learning Representations (ICLR)},
  2018.

\bibitem{deng2009imagenet}
J.~Deng, W.~Dong, R.~Socher, L.-J. Li, K.~Li, and L.~Fei-Fei.
\newblock Imagenet: A large-scale hierarchical image database.
\newblock In {\em 2009 IEEE conference on computer vision and pattern
  recognition}, pages 248--255. Ieee, 2009.

\bibitem{duan2012domain}
L.~Duan, I.~W. Tsang, and D.~Xu.
\newblock Domain transfer multiple kernel learning.
\newblock {\em IEEE Transactions on Pattern Analysis and Machine Intelligence},
  34(3):465--479, 2012.

\bibitem{duan2012visual}
L.~Duan, D.~Xu, I.~W.-H. Tsang, and J.~Luo.
\newblock Visual event recognition in videos by learning from web data.
\newblock {\em IEEE Transactions on Pattern Analysis and Machine Intelligence},
  34(9):1667--1680, 2012.

\bibitem{felzenszwalb2010object}
P.~F. Felzenszwalb, R.~B. Girshick, D.~McAllester, and D.~Ramanan.
\newblock Object detection with discriminatively trained part-based models.
\newblock {\em IEEE transactions on pattern analysis and machine intelligence},
  32(9):1627--1645, 2010.

\bibitem{fernando2013unsupervised}
B.~Fernando, A.~Habrard, M.~Sebban, and T.~Tuytelaars.
\newblock Unsupervised visual domain adaptation using subspace alignment.
\newblock In {\em Proceedings of the IEEE international conference on computer
  vision}, pages 2960--2967, 2013.

\bibitem{ganin2014unsupervised}
Y.~Ganin and V.~Lempitsky.
\newblock Unsupervised domain adaptation by backpropagation.
\newblock {\em arXiv preprint arXiv:1409.7495}, 2014.

\bibitem{ganin2016domain}
Y.~Ganin, E.~Ustinova, H.~Ajakan, P.~Germain, H.~Larochelle, F.~Laviolette,
  M.~Marchand, and V.~Lempitsky.
\newblock Domain-adversarial training of neural networks.
\newblock {\em The Journal of Machine Learning Research}, 17(1):2096--2030,
  2016.

\bibitem{gebru2017fine}
T.~Gebru, J.~Hoffman, and L.~Fei-Fei.
\newblock Fine-grained recognition in the wild: A multi-task domain adaptation
  approach.
\newblock In {\em Computer Vision (ICCV), 2017 IEEE International Conference
  on}, pages 1358--1367. IEEE, 2017.

\bibitem{Geiger2013IJRR}
A.~Geiger, P.~Lenz, C.~Stiller, and R.~Urtasun.
\newblock Vision meets robotics: The kitti dataset.
\newblock {\em International Journal of Robotics Research (IJRR)}, 2013.

\bibitem{ghifary2016deep}
M.~Ghifary, W.~B. Kleijn, M.~Zhang, D.~Balduzzi, and W.~Li.
\newblock Deep reconstruction-classification networks for unsupervised domain
  adaptation.
\newblock In {\em European Conference on Computer Vision}, pages 597--613.
  Springer, 2016.

\bibitem{girshick2015fast}
R.~Girshick.
\newblock Fast r-cnn.
\newblock pages 1440--1448, 2015.

\bibitem{girshick2014rich}
R.~Girshick, J.~Donahue, T.~Darrell, and J.~Malik.
\newblock Rich feature hierarchies for accurate object detection and semantic
  segmentation.
\newblock In {\em Proceedings of the IEEE conference on computer vision and
  pattern recognition}, pages 580--587, 2014.

\bibitem{gong2012geodesic}
B.~Gong, Y.~Shi, F.~Sha, and K.~Grauman.
\newblock Geodesic flow kernel for unsupervised domain adaptation.
\newblock In {\em Computer Vision and Pattern Recognition (CVPR), 2012 IEEE
  Conference on}, pages 2066--2073. IEEE, 2012.

\bibitem{gopalan2011domain}
R.~Gopalan, R.~Li, and R.~Chellappa.
\newblock Domain adaptation for object recognition: An unsupervised approach.
\newblock In {\em Computer Vision (ICCV), 2011 IEEE International Conference
  on}, pages 999--1006. IEEE, 2011.

\bibitem{he2014spatial}
K.~He, X.~Zhang, S.~Ren, and J.~Sun.
\newblock Spatial pyramid pooling in deep convolutional networks for visual
  recognition.
\newblock In {\em European conference on computer vision}, pages 346--361.
  Springer, 2014.

\bibitem{hoffman2017cycada}
J.~Hoffman, E.~Tzeng, T.~Park, J.-Y. Zhu, P.~Isola, K.~Saenko, A.~A. Efros, and
  T.~Darrell.
\newblock Cycada: Cycle-consistent adversarial domain adaptation.
\newblock {\em arXiv preprint arXiv:1711.03213}, 2017.

\bibitem{hoffman2016fcns}
J.~Hoffman, D.~Wang, F.~Yu, and T.~Darrell.
\newblock Fcns in the wild: Pixel-level adversarial and constraint-based
  adaptation.
\newblock {\em arXiv preprint arXiv:1612.02649}, 2016.

\bibitem{huang2017speed}
J.~Huang, V.~Rathod, C.~Sun, M.~Zhu, A.~Korattikara, A.~Fathi, I.~Fischer,
  Z.~Wojna, Y.~Song, S.~Guadarrama, et~al.
\newblock Speed/accuracy trade-offs for modern convolutional object detectors.
\newblock In {\em IEEE CVPR}, volume~4, 2017.

\bibitem{inoue2018cross}
N.~Inoue, R.~Furuta, T.~Yamasaki, and K.~Aizawa.
\newblock Cross-domain weakly-supervised object detection through progressive
  domain adaptation.
\newblock In {\em Proceedings of the IEEE Conference on Computer Vision and
  Pattern Recognition}, pages 5001--5009, 2018.

\bibitem{JiangICML18mentornet}
L.~Jiang, Z.~Zhou, T.~Leung, L.-J. Li, and L.~Fei-Fei.
\newblock Mentornet: Regularizing very deep neural networks on corrupted
  labels.
\newblock In {\em International Conference on Machine Learning (ICML)}, 2018.

\bibitem{johnson2016driving}
M.~Johnson-Roberson, C.~Barto, R.~Mehta, S.~N. Sridhar, K.~Rosaen, and
  R.~Vasudevan.
\newblock Driving in the matrix: Can virtual worlds replace human-generated
  annotations for real world tasks?
\newblock {\em arXiv preprint arXiv:1610.01983}, 2016.

\bibitem{kulis2011you}
B.~Kulis, K.~Saenko, and T.~Darrell.
\newblock What you saw is not what you get: Domain adaptation using asymmetric
  kernel transforms.
\newblock In {\em Computer Vision and Pattern Recognition (CVPR), 2011 IEEE
  Conference on}, pages 1785--1792. IEEE, 2011.

\bibitem{li2018domain}
W.~Li, Z.~Xu, D.~Xu, D.~Dai, and L.~Van~Gool.
\newblock Domain generalization and adaptation using low rank exemplar svms.
\newblock {\em IEEE transactions on pattern analysis and machine intelligence},
  40(5):1114--1127, 2018.

\bibitem{liu2016ssd}
W.~Liu, D.~Anguelov, D.~Erhan, C.~Szegedy, S.~Reed, C.-Y. Fu, and A.~C. Berg.
\newblock Ssd: Single shot multibox detector.
\newblock In {\em European conference on computer vision}, pages 21--37.
  Springer, 2016.

\bibitem{long2015learning}
M.~Long, Y.~Cao, J.~Wang, and M.~I. Jordan.
\newblock Learning transferable features with deep adaptation networks.
\newblock {\em arXiv preprint arXiv:1502.02791}, 2015.

\bibitem{long2016unsupervised}
M.~Long, H.~Zhu, J.~Wang, and M.~I. Jordan.
\newblock Unsupervised domain adaptation with residual transfer networks.
\newblock In {\em Advances in Neural Information Processing Systems}, pages
  136--144, 2016.

\bibitem{long2017deep}
M.~Long, H.~Zhu, J.~Wang, and M.~I. Jordan.
\newblock Deep transfer learning with joint adaptation networks.
\newblock In {\em Proceedings of the 34th International Conference on Machine
  Learning-Volume 70}, pages 2208--2217. JMLR. org, 2017.

\bibitem{MisraCVPR16LabelingBias}
I.~Misra, C.~Lawrence~Zitnick, M.~Mitchell, and R.~Girshick.
\newblock Seeing through the human reporting bias: Visual classifiers from
  noisy human-centric labels.
\newblock In {\em CVPR}, 2016.

\bibitem{MnihICML12Arial}
V.~Mnih and G.~E. Hinton.
\newblock Learning to label aerial images from noisy data.
\newblock In {\em International Conference on Machine Learning (ICML)}, pages
  567--574, 2012.

\bibitem{motiian2017unified}
S.~Motiian, M.~Piccirilli, D.~A. Adjeroh, and G.~Doretto.
\newblock Unified deep supervised domain adaptation and generalization.
\newblock In {\em The IEEE International Conference on Computer Vision (ICCV)},
  volume~2, page~3, 2017.

\bibitem{murez2018image}
Z.~Murez, S.~Kolouri, D.~Kriegman, R.~Ramamoorthi, and K.~Kim.
\newblock Image to image translation for domain adaptation.
\newblock In {\em Proceedings of the IEEE Conference on Computer Vision and
  Pattern Recognition}, pages 4500--4509, 2018.

\bibitem{NatarajanNIPS13noisy}
N.~Natarajan, I.~S. Dhillon, P.~K. Ravikumar, and A.~Tewari.
\newblock Learning with noisy labels.
\newblock In {\em Advances in neural information processing systems}, pages
  1196--1204, 2013.

\bibitem{PatriniCVPR17}
G.~Patrini, A.~Rozza, A.~Menon, R.~Nock, and L.~Qu.
\newblock Making neural networks robust to label noise: {A} loss correction
  approach.
\newblock In {\em Computer Vision and Pattern Recognition}, 2017.

\bibitem{raj2015subspace}
A.~Raj, V.~P. Namboodiri, and T.~Tuytelaars.
\newblock Subspace alignment based domain adaptation for rcnn detector.
\newblock {\em arXiv preprint arXiv:1507.05578}, 2015.

\bibitem{redmon2016you}
J.~Redmon, S.~Divvala, R.~Girshick, and A.~Farhadi.
\newblock You only look once: Unified, real-time object detection.
\newblock In {\em Proceedings of the IEEE conference on computer vision and
  pattern recognition}, pages 779--788, 2016.

\bibitem{redmon2017yolo9000}
J.~Redmon and A.~Farhadi.
\newblock Yolo9000: better, faster, stronger.
\newblock {\em arXiv preprint}, 2017.

\bibitem{RenICML18Robust}
M.~Ren, W.~Zeng, B.~Yang, and R.~Urtasun.
\newblock Learning to reweight examples for robust deep learning.
\newblock In {\em International Conference on Machine Learning (ICML)}, 2018.

\bibitem{ren2015faster}
S.~Ren, K.~He, R.~Girshick, and J.~Sun.
\newblock Faster r-cnn: Towards real-time object detection with region proposal
  networks.
\newblock In {\em Advances in neural information processing systems}, pages
  91--99, 2015.

\bibitem{sakaridis2018semantic}
C.~Sakaridis, D.~Dai, and L.~Van~Gool.
\newblock Semantic foggy scene understanding with synthetic data.
\newblock {\em International Journal of Computer Vision}, pages 1--20, 2018.

\bibitem{Sukhbaatar14Noisy}
S.~Sukhbaatar, J.~Bruna, M.~Paluri, L.~Bourdev, and R.~Fergus.
\newblock Training convolutional networks with noisy labels.
\newblock {\em arXiv preprint arXiv:1406.2080}, 2014.

\bibitem{sun2016return}
B.~Sun, J.~Feng, and K.~Saenko.
\newblock Return of frustratingly easy domain adaptation.
\newblock In {\em AAAI}, volume~6, page~8, 2016.

\bibitem{szegedy2017inception}
C.~Szegedy, S.~Ioffe, V.~Vanhoucke, and A.~A. Alemi.
\newblock Inception-v4, inception-resnet and the impact of residual connections
  on learning.
\newblock In {\em Thirty-First AAAI Conference on Artificial Intelligence},
  2017.

\bibitem{szegedy2016rethinking}
C.~Szegedy, V.~Vanhoucke, S.~Ioffe, J.~Shlens, and Z.~Wojna.
\newblock Rethinking the inception architecture for computer vision.
\newblock In {\em Proceedings of the IEEE conference on computer vision and
  pattern recognition}, pages 2818--2826, 2016.

\bibitem{TanakaCVPR18Noisy}
D.~Tanaka, D.~Ikami, T.~Yamasaki, and K.~Aizawa.
\newblock Joint optimization framework for learning with noisy labels.
\newblock In {\em Computer Vision and Pattern Recognition (CVPR)}, 2018.

\bibitem{tzeng2017adversarial}
E.~Tzeng, J.~Hoffman, K.~Saenko, and T.~Darrell.
\newblock Adversarial discriminative domain adaptation.
\newblock In {\em Proceedings of the IEEE Conference on Computer Vision and
  Pattern Recognition}, pages 7167--7176, 2017.

\bibitem{uijlings2013selective}
J.~R. Uijlings, K.~E. Van De~Sande, T.~Gevers, and A.~W. Smeulders.
\newblock Selective search for object recognition.
\newblock {\em International journal of computer vision}, 104(2):154--171,
  2013.

\bibitem{VahdatNIPS17Robust}
A.~Vahdat.
\newblock Toward robustness against label noise in training deep discriminative
  neural networks.
\newblock In {\em Neural Information Processing Systems (NIPS)}, 2017.

\bibitem{VahdatM13}
A.~Vahdat and G.~Mori.
\newblock Handling uncertain tags in visual recognition.
\newblock In {\em International Conference on Computer Vision (ICCV)}, 2013.

\bibitem{VahdatZM14}
A.~Vahdat, G.-T. Zhou, and G.~Mori.
\newblock Discovering video clusters from visual features and noisy tags.
\newblock In {\em European Conference on Computer Vision (ECCV)}, 2014.

\bibitem{VeitCVPR17Noisy}
A.~Veit, N.~Alldrin, G.~Chechik, I.~Krasin, A.~Gupta, and S.~Belongie.
\newblock Learning from noisy large-scale datasets with minimal supervision.
\newblock In {\em 2017 IEEE Conference on Computer Vision and Pattern
  Recognition (CVPR)}, pages 6575--6583. IEEE, 2017.

\bibitem{viola2001rapid}
P.~Viola and M.~Jones.
\newblock Rapid object detection using a boosted cascade of simple features.
\newblock In {\em Computer Vision and Pattern Recognition, 2001. CVPR 2001.
  Proceedings of the 2001 IEEE Computer Society Conference on}, volume~1, pages
  I--I. IEEE, 2001.

\bibitem{wang2017deep}
Y.~Wang, W.~Li, D.~Dai, and L.~Van~Gool.
\newblock Deep domain adaptation by geodesic distance minimization.
\newblock {\em arXiv preprint arXiv:1707.09842}, 2017.

\bibitem{XiaoCVPR15}
T.~Xiao, T.~Xia, Y.~Yang, C.~Huang, and X.~Wang.
\newblock Learning from massive noisy labeled data for image classification.
\newblock In {\em Computer Vision and Pattern Recognition (CVPR)}, 2015.

\bibitem{xu2014domain}
J.~Xu, S.~Ramos, D.~V{\'a}zquez, and A.~M. L{\'o}pez.
\newblock Domain adaptation of deformable part-based models.
\newblock {\em IEEE transactions on pattern analysis and machine intelligence},
  36(12):2367--2380, 2014.

\bibitem{yu2018learning}
X.~Yu, T.~Liu, M.~Gong, and D.~Tao.
\newblock Learning with biased complementary labels.
\newblock In {\em Proceedings of the European Conference on Computer Vision
  (ECCV)}, pages 68--83, 2018.

\bibitem{zhang2016faster}
L.~Zhang, L.~Lin, X.~Liang, and K.~He.
\newblock Is faster r-cnn doing well for pedestrian detection?
\newblock In {\em European Conference on Computer Vision}, pages 443--457.
  Springer, 2016.

\bibitem{zhang2017curriculum}
Y.~Zhang, P.~David, and B.~Gong.
\newblock Curriculum domain adaptation for semantic segmentation of urban
  scenes.
\newblock In {\em The IEEE International Conference on Computer Vision (ICCV)},
  volume~2, page~6, 2017.

\bibitem{ZhangNIPS18generalized}
Z.~Zhang and M.~R. Sabuncu.
\newblock Generalized cross entropy loss for training deep neural networks with
  noisy labels.
\newblock In {\em Neural Information Processing Systems (NIPS)}, 2018.

\bibitem{zhu2017unpaired}
J.-Y. Zhu, T.~Park, P.~Isola, and A.~A. Efros.
\newblock Unpaired image-to-image translation using cycle-consistent
  adversarial networks.
\newblock In {\em Proceedings of the IEEE International Conference on Computer
  Vision}, pages 2223--2232, 2017.

\end{thebibliography}
}

\end{document}